\documentclass[letterpaper,10pt,conference]{ieeeconf}

\IEEEoverridecommandlockouts	%
\makeatletter

\let\proof\@undefined
\let\endproof\@undefined
\makeatother

\usepackage[abs]{overpic}
\usepackage{amsmath,amssymb,amsthm}
\usepackage{graphicx,url}
\usepackage{epsfig,psfrag,color}
\usepackage{xspace}
\usepackage{caption}

\usepackage{caption}

\usepackage[linesnumbered,vlined,ruled]{algorithm2e}
\usepackage{multirow}
\usepackage[normalem]{ulem} %
\usepackage{cancel}
\usepackage{subfigure}
\usepackage[version=0.96]{pgf}
\usepackage{tikz}
\usetikzlibrary{arrows,automata}

\newtheorem{theorem}{Theorem}[section]
\newtheorem{proposition}[theorem]{Proposition}
\newtheorem{corollary}[theorem]{Corollary}
\newtheorem{definition}[theorem]{Definition}

\newtheorem{remark}[theorem]{Remark}

\newtheorem{example}[theorem]{Example}

\newtheorem{problem}[theorem]{Problem}

\newcommand{\RMA}{\mathrm{A}}

\newcommand{\RMF}{\mathrm{F}}
\newcommand{\RMM}{\mathrm{M}}
\newcommand{\RMP}{\mathrm{P}}

\newcommand{\RMT}{\mathrm{T}}

\DeclareMathAlphabet{\mathpzc}{OT1}{pzc}{m}{it}
\newcommand{\CA}{\mathcal{A}}

\newcommand{\CE}{\mathcal{E}}
\newcommand{\CF}{\mathcal{F}}

\newcommand{\CL}{\mathcal{L}}
\newcommand{\CM}{\mathcal{M}}

\newcommand{\CQ}{\mathcal{Q}}

\newcommand{\BFA}{\mathbf{A}}

\newcommand{\BFF}{\mathbf{F}}

\newcommand{\BFM}{\mathbf{M}}
\newcommand{\BFP}{\mathbf{P}}

\newcommand{\BFT}{\mathbf{T}}
\newcommand{\BFV}{\mathbf{V}}

\graphicspath{{fig/}}

\newcommand{\notltl}{\neg}
\newcommand{\andltl}{\wedge}
\newcommand{\orltl}{\vee}

\newcommand{\Next}{\mathbf{X}}

\newcommand{\Event}{\mathbf{F}}
\newcommand{\Until}{\mathcal{U}}

\newcommand{\True}{\top}
\newcommand{\prop}{\mathtt{p}}

\newcommand{\ie}{{\it i.e.},\;}
\newcommand{\eg}{{\it e.g.},\;}

\newcommand\oprocendsymbol{\hbox{$\square$}}
\newcommand\oprocend{\relax\ifmmode\else\unskip\hfill\fi\oprocendsymbol}

\urldef{\ulusoy}\url{alphan@bu.edu}
\urldef{\nok}\url{nok@smart.mit.edu}
\urldef{\belta}\url{cbelta@bu.edu}
\title{\LARGE \bf Incremental Control Synthesis in Probabilistic Environments with Temporal Logic Constraints}
\author{Alphan Ulusoy$^\star$%
 		\quad Tichakorn Wongpiromsarn$^\dagger$%
 		\quad Calin Belta$^\star$%
 		\thanks{This work was supported in part by ONR-MURI N00014-09-1051, ONR MURI N00014-10-10952, NSF CNS-0834260 and NSF CNS-1035588.}%
 		\thanks{$^\star$ Division of Systems Engineering, Boston University, Boston, MA 02215 (\ulusoy, \belta)}%
 		\thanks{$^\dagger$ Singapore-MIT Alliance for Research and Technology, Singapore 117543 (\nok)}}%

\renewcommand{\langle}{(}
\renewcommand{\rangle}{)}

\begin{document}

\maketitle
\thispagestyle{empty}
\pagestyle{empty}

\begin{abstract}
In this paper, we present a method for optimal control synthesis of a plant that
interacts with a set of agents in a graph-like environment. The control specification is given as a temporal logic statement about some properties that hold at the vertices of the environment. The plant is assumed to be deterministic, while the agents are probabilistic Markov models. The goal is to control the plant such that the probability of satisfying a syntactically co-safe Linear Temporal Logic formula is maximized. 
We propose a computationally efficient incremental approach based on the fact that temporal logic verification is computationally cheaper than synthesis.
We present a case-study where we compare our approach to the classical non-incremental approach in terms of computation time and memory usage.
\end{abstract}

\section{Introduction\label{sec:introduction}} 
Temporal logics~\cite{EAE:90}, such as Linear Temporal Logic (LTL) and Computation Tree Logic (CTL), are traditionally used for verification of non-deterministic and probabilistic systems~\cite{CB-JPK:08}.
Even though temporal logics are suitable for specifying complex missions for control systems, they did not gain popularity in the control community until recently~\cite{PT-GJP:06, MK-CB:10, TW-UT-RMM:10}.

The existing works on control synthesis focus on specifications given in linear time temporal logic.
The systems, which sometimes are obtained through an additional abstraction process \cite{PT-GJP:06,TuYoBeCeBa-CDC10}, have finitely many states. 
With few exceptions \cite{Chatterjee2012}, their states are fully observable. 
For such systems, control strategies can be synthesized through exhaustive search of the state space.
If the system is deterministic, model checking tools can be easily adapted to generate control strategies~\cite{MK-CB:10}.
If the system is non-deterministic, the control problem can be mapped to the solution of a Rabin game~\cite{WT:02,TuYoBeCeBa-CDC10}, or a simpler B\"uchi \cite{KlBe-HSCC08-book} or GR(1) game~\cite{HKG-GF-GFP:07}, if the specification is restricted to fragments of LTL.
For probabilistic systems, the LTL control synthesis problem reduces to computing a control policy for a Markov Decision Process (MDP)~\cite{AB-LDA:95,MK-GN-DP:02,DiSmBeRu-CDC-2011}.

In this work, we consider mission specifications expressed as syntactically co-safe LTL formulas \cite{OK-MYV:01}.
We focus on a particular type of a multi-agent system formed by a deterministically controlled plant and a set of independent, probabilistic, uncontrollable agents, operating on a common, graph-like environment. 
An illustrative example is a car (plant) approaching a pedestrian crossing, while there are some pedestrians (agents) waiting to cross or already crossing the road.
As the state space of the system grows exponentially with the number of pedestrians, one may not be able to utilize any of the existing approaches under computational resource constraints when there is a large number of pedestrians.

We partially address this problem by proposing an incremental control synthesis method that exploits the independence between the components of the system, \ie the plant modeled as a deterministic transition system and the agents, modeled as Markov chains, and the fact that verification is computationally cheaper than synthesis.
We aim to synthesize a plant control strategy that maximize the probability of satisfying a mission specification given as a syntactically co-safe LTL formula.
Our method initially considers a considerably smaller agent subset and synthesizes a control policy that maximizes the probability of satisfying the mission specification for the subsystem formed by the plant and this subset.
This control policy is then verified against the remaining agents.
At each iteration, we remove transitions and states that are not needed in subsequent iterations. This leads to a significant reduction in computation time and memory usage.
It is important to note that our method does not need to run to completion.
A sub-optimal control policy can be obtained by forcing termination at a given iteration if the computation is performed under stringent resource constraints. 
It must also be noted that our framework easily extends to the case when the plant is a Markov Decision Process, and we consider a deterministic plant only for simplicity of presentation.
We experimentally evaluate the performance of our approach and show that our method clearly outperforms existing non-incremental approaches.
Various methods that also use verification during incremental synthesis have been previously proposed in \cite{SJ-SG-SS-AT:2010,SG-SJ-SS-AT-RV:2011}.
However, the approach that we present in this paper is, to the best of our knowledge, the first use of verification guided incremental synthesis in the context of probabilistic systems.

The rest of the paper is organized as follows:
In Sec.~\ref{sec:preliminaries}, we give necessary definitions and some preliminaries in formal methods. The control synthesis problem is formally stated in Sec.~\ref{sec:problem} and the solution is presented in Sec.~\ref{sec:sol}.
Experimental results are included in Sec.~\ref{sec:impl}. 
We conclude with final remarks in Sec.~\ref{sec:conclusions}.

\section{Preliminaries \label{sec:preliminaries}}

For a set $\Sigma$, we use $|\Sigma|$ and $2^\Sigma$ to denote its cardinality and power set, respectively.
A (finite) word $\omega$ over a set $\Sigma$ is a sequence of symbols $\omega = \omega^0\ldots\omega^l$ such that $\omega^i \in \Sigma\;\forall i = 0,\ldots,l$.

\begin{definition}[\bf Transition System]
\label{def:TS}
A transition system (TS) is a tuple $\BFT := (\CQ_\RMT, q_\RMT^0, \CA_\RMT, \alpha_\RMT, \delta_\RMT, \Pi_\RMT, \CL_\RMT)$, where
\begin{itemize}
\item $\CQ_\RMT$ is a finite set of states; %
\item $q_\RMT^0 \in \CQ_\RMT$ is the initial state; %
\item $\CA_\RMT$ is a finite set of actions; %
\item $\alpha_\RMT: \CQ_\RMT \to 2^{\CA_\RMT}$ is a map giving the set of actions available at a state;%
\item $\delta_\RMT \subseteq \CQ_\RMT \times \CA_\RMT \times \CQ_\RMT$ is the transition relation; %
\item $\Pi_\RMT$ is a finite set of atomic propositions; %
\item $\CL_\RMT:\CQ_\RMT\to 2^{\Pi_\RMT}$ is a satisfaction map giving the set of atomic propositions satisfied at a state.%
\end{itemize}
\end{definition}

\begin{definition}[\bf Markov Chain]
\label{def:MC}
A (discrete-time, labelled) Markov chain (MC) is a tuple $\BFM := (\CQ_\RMM, q_\RMM^0, \delta_\RMM, \Pi_\RMM, \CL_\RMM)$, where $\CQ_\RMM$, $\Pi_\RMM$, and $\CL_\RMM$ are the set of states, the set of atomic propositions, and the satisfaction map, respectively, as in Def.~\ref{def:TS}, and
\begin{itemize}
\item $q_\RMM^0 \in \CQ_\RMM$ is the initial state; %
\item $\delta_\RMM: \CQ_\RMM \times \CQ_\RMM \to [0,1]$ is the transition probability function that satisfies $\sum_{q'\in\CQ_\RMM} \delta(q,q') = 1\; \forall q\in\CQ_\RMM$. %
\end{itemize}
\end{definition}

In this paper, we are interested in temporal logic missions over a finite time horizon and we use syntactically co-safe LTL formulas~\cite{AB-LEK-MYV:10} to specify them.
Informally, a syntactically co-safe LTL formula over the set $\Pi$ of atomic propositions comprises boolean operators $\notltl$ (negation), $\orltl$ (disjunction) and $\andltl$ (conjunction), and temporal operators $\Next$ (next), $\Until$ (until) and $\Event$ (eventually).
Any syntactically co-safe LTL formula can be written in positive normal form, where the negation operator $\notltl$ occurs only in front of atomic propositions.
For instance, $\Next\,\prop$ states that at the next position of the word, proposition $\prop$ is true.
The formula $\prop_1\,\Until\,\prop_2$ states that there is a future position of the word when proposition $\prop_2$ is true, and proposition $\prop_1$ is true at least until $\prop_2$ is true.
For any syntactically co-safe LTL formula $\phi$ over a set $\Pi$, one can construct a FSA with input alphabet $2^{\Pi}$ accepting all and only finite words over $2^\Pi$ that satisfy $\phi$, which is defined next.

\begin{definition}[\bf Finite State Automaton]
\label{def:buchi}
A (deterministic) finite state automaton (FSA) is a tuple $\BFF := (\CQ_\RMF,q_\RMF^0,\Sigma_\RMF,\delta_\RMF,\CF_\RMF)$, where %
\begin{itemize}
\item $\CQ_\RMF$ is a finite set of states; %
\item $q_\RMF^0 \in \CQ_\RMF$ is the initial state; %
\item $\Sigma_\RMF$ is an input alphabet; %
\item $\delta_\RMF: \CQ_\RMF \times \Sigma_\RMF \times \CQ_\RMF$ is a deterministic transition relation; %
\item $\CF_\RMF\subseteq \CQ_\RMF$ is a set of accepting (final) states.
\end{itemize}
\end{definition}
A \emph{run} of $\BFF$ over an input word $\omega=\omega^0\omega^1\ldots\omega^l$ where $\omega^i\in\Sigma_\RMF\;\forall i=0\ldots l$ is a sequence $r_\RMF=q^0 q^1\ldots q^lq^{l+1}$, such that $(q^i,\omega^i,q^{i+1}) \in \delta_\RMF\;\forall i=0\ldots l$ and $q^0=q^0_\RMF$.
An FSA $\BFF$ accepts a word over $\Sigma_\RMF$ if and only the corresponding run ends in some $q \in \CF_\RMF$.

\begin{definition}[\bf Markov Decision Process]
\label{def:MDP}
A Markov decision process (MDP) is a tuple $\BFP := (\CQ_\RMP, q_\RMP^0, \CA_\RMP, \alpha_\RMP, \delta_\RMP, \Pi_\RMP, \CL_\RMP)$, where
\begin{itemize}
\item $\CQ_\RMP$ is a finite set of states; %
\item $q_\RMP^0 \in \CQ_\RMP$ is the initial state; %
\item $\CA_\RMP$ is a finite set of actions; %
\item $\alpha_\RMP: \CQ_\RMP \to 2^{\CA_\RMP}$ is a map giving the set of actions available at a state;%
\item $\delta_\RMP: \CQ_\RMP \times \CA_\RMP \times \CQ_\RMP \to [0,1]$ is the transition probability function that satisfies $\sum_{q'\in\CQ_\RMP} \delta(q,a,q') = 1\; \forall q\in\CQ_\RMP, a\in\alpha_\RMP(q)$ and $\sum_{q'\in\CQ_\RMP} \delta(q,a,q') = 0\; \forall q\in\CQ_\RMP, a\not\in\alpha_\RMP(q)$.%
\item $\Pi_\RMP$ is a finite set of atomic propositions; %
\item $\CL_\RMP:\CQ_\RMP\to 2^{\Pi_\RMP}$ is a map giving the set of atomic propositions satisfied in a state.%
\end{itemize}
\end{definition}
For an MDP $\BFP$, we define a stationary policy $\mu_\RMP: \CQ_\RMP \to \CA_\RMP$ such that for a state $q \in \CQ_\RMP$, $\mu_\RMP(q) \in \alpha_\RMP(q)$.
This stationary policy can then be used to resolve all nondeterministic choices in $\BFP$ by applying action $\mu(q)$ at each $q \in \CQ_\RMP$.
A path of $\BFP$ under policy $\mu_\RMP$ is a finite sequence of states $r_\RMP^{\mu_\RMP} = q^0q^1\ldots q^l$ such that $l\geq0$, $q^0 = q^0_\RMP$ and $\delta_\RMP(q^{k-1},\mu_\RMP(q^{k-1}),q^{k})>0$ $\forall k\in [1,l]$.
A path $r_\RMP^{\mu_\RMP}$ generates a finite word $\CL_\RMP(r_\RMP^{\mu_\RMP}) = \CL_\RMP(q^0)\CL_\RMP(q^1)\ldots\CL_\RMP(q^l)$ where $\CL_\RMP(q^k)$ is the set of atomic propositions satisfied at state $q^k$.
Next, we use $Paths_\RMP^{\mu_\RMP}$ to denote the set of all paths of $\BFP$ under a policy $\mu_\RMP$.
Finally, we define $P^{\mu_\RMP}(\phi)$ as the probability of satisfying $\phi$ under policy $\mu_\RMP$.
\begin{remark}
Syntactically co-safe LTL formulas have infinite time semantics, thus they are actually interpreted over infinite words~\cite{AB-LEK-MYV:10}.
Measurability of languages satisfying LTL formulas is also defined for infinite words generated by infinite paths~\cite{CB-JPK:08}.
However, one can determine whether a given infinite word satisfies a syntactically co-safe LTL formula by considering only a finite prefix of it.
It can be easily shown that our above definition of $Paths^{\mu_\RMP}_\RMP$ inherits the same measurability property given in~\cite{CB-JPK:08}.
\end{remark}

\section{Problem Formulation and Approach \label{sec:problem}}

In this section we introduce the control synthesis problem with temporal constraints for a system that models a plant operating in the presence of probabilistic independent agents.

\subsection{System Model}

Consider a system consisting of a deterministic plant that we can control (\eg a robot) and $n$ agents operating in an environment modeled by a graph $\CE = (V, \rightarrow_\CE,\CL_\CE,\Pi_\CE)$, where $V$ is the set of vertices, $\rightarrow_\CE \subseteq V\times V$ is the set of edges, and $\CL_\CE$ is the labeling function that maps each vertex to a proposition in $\Pi_\CE$.
For example, $\CE$ can be the quotient graph of a partitioned environment, where $V$ is a set of labels for the regions in the partition and $\rightarrow_\CE$ is the corresponding adjacency relation (see Figs.~\ref{fig:drawing},~\ref{fig:simple_models}).
Agent $i$ is modeled as an MC $\BFM_i = (\CQ_i, q_i^0, \delta_i, \Pi_i, \CL_i)$, with  $\CQ_i\subseteq V$ and $\delta_i\subseteq\rightarrow_\CE$, $i=1,\ldots,n$. 
The plant is assumed to be a deterministic transition system TS $\BFT = (\CQ_\RMT,q_\RMT^0,\CA_\RMT,\alpha_\RMT,\delta_\RMT,\Pi_\RMT,\CL_\RMT)$, where $\CQ_\RMT\subseteq V$ and $\delta_t\subseteq\rightarrow_\CE$.
We assume that all components of the system (the plant and the agents) make transitions synchronously by picking edges of the graph.
We also assume that the state of the system is perfectly known at any given instant and we can control the plant but we have no control over the agents.

\begin{figure}
\centering 
	\includegraphics{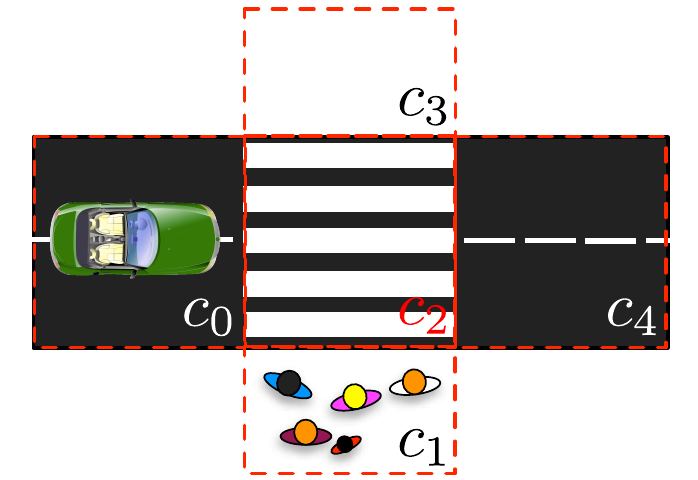}
   \caption{A partitioned road environment, where a car (plant) is required to reach $c_4$ without colliding with any of the pedestrians (agents).}
   \label{fig:drawing}
\end{figure}

We define the sets of propositions and labeling functions of the individual components of the system such that they inherit the propositions of their current vertex from the graph while preserving their own identities. Formally, we have $\Pi_\RMT=\{(\RMT,\CL_\CE(q)) | q\in \CQ_\RMT\}$ and $\CL_\RMT(q)=(\RMT,\CL_\CE(q))$ for the plant, and $\Pi_i=\{(i,\CL_\CE(q))|q\in \CQ_i\}$ and $\CL_i(q)=(i,\CL_\CE(q))$ for agent $i$.
Finally, we define the set $\Pi$ of propositions as $\Pi=\Pi_\RMT\cup\Pi_i\cup\ldots\cup\Pi_n\subseteq\{(i,p) | i = \{\RMT,0,\ldots,n\}, p\in\Pi_\CE\}$.

\subsection{Problem Formulation\label{sec:sub:prob}}

As it will become clear in Sec. \ref{sec:sub:syn}, the joint behavior of the plant and agents in the graph environment can be modeled by the parallel composition of the TS and MC models described above, which takes the form of an MDP (see Def. \ref{def:MDP}). 
Given a syntactically co-safe LTL formula $\phi$ over $\Pi$, our goal is to synthesize a policy for this MDP, which we will simply refer to as the {\it system}, such that the probability of satisfying $\phi$ is either maximized or above a given threshold.  
Since we assume perfect state information, the plant can implement a control policy computed for the system, i.e,  based on
its state and the state of all the other agents. As a result, we will not distinguish between a control policy for the plant and a control policy for the system, and we will refer to it simply as {\em control policy}.
We can now formulate the main problem considered in this paper:

\begin{problem}
\label{prb:problem}
Given a system described by a plant $\BFT$ and a set of agents $\BFM_1,\ldots,\BFM_n$ operating on a graph $\mathcal{E}$, and given a specification in the form of a syntactically co-safe LTL formula $\phi$ over $\Pi$, synthesize a control policy $\mu^\star$ that satisfies the following objective: (a) If a probability threshold $p_{thr}$ is given, the probability that the system satisfies $\phi$ under $\mu^\star$ exceeds $p_{thr}$. (b) Otherwise, $\mu^\star$ maximizes the probability that the system satisfies $\phi$.
If no such policy exists, report failure.
\end{problem}

As will be shown in Sec.~\ref{sec:sub:comp}, the parallel composition of MDP and MC models also takes the form of an MDP.
Hence, our approach can easily accommodate the case where the plant is a Markov Decision Process.
We consider a deterministic plant only for simplicity of presentation.

\begin{example}
\label{exp:simple}
Fig.~\ref{fig:drawing} illustrates a car in a 5-cell environment with 5 pedestrians, where $\CL_\CE(v)=v$ for $v\in\{c_0,\ldots,c_4\}$.
Fig.~\ref{fig:simple_models} illustrates the TS $\BFT$ and the MCs $\BFM_1,\ldots,\BFM_5$ that model the car and the pedestrians.
The car is required to reach the end of the crossing ($c_4$) without colliding with any of the pedestrians.
To enforce this behavior, we write our specification as 
\begin{equation}
\label{eqn:phi}
\phi := \left(\notltl \bigvee_{i=1\ldots5,j=0\ldots4} ((\RMT,c_j)\andltl (i,c_j)) \right) \,\Until\,(\RMT,c_4).
\end{equation}
The deterministic FSA that corresponds to $\phi$ is given in Fig.~\ref{fig:simple_fsa}, where $\mathtt{col}=\bigvee_{i=1\ldots5,j=0\ldots4} ((\RMT,c_j)\andltl (i,c_j))$ and $\mathtt{end} = (\RMT,c_4)$.
\end{example}

\begin{figure}
\subfigure{
	\begin{tikzpicture}[->,>=stealth',shorten >=1pt,auto,node distance=2cm]
	  \tikzstyle{every state}=[circle,thick,draw=blue!75,minimum size=6mm]

	  \node [initial,initial text=,state] (c0) at (0,2.4) {$c_0$};
	   \node [state] (c2) at (0,1.2){$c_2$};
	   \node [state,label=left:$\mathtt{end}$,label=below:$\BFT$] (c4) at (0,0) {$c_4$};
	   
	   \path (c0) edge [loop right] node {wait} (c0)
	                     edge node [] {go} (c2)
	             (c2) edge [loop right] node {wait} (c2)
	                     edge node [] {go} (c4)
						 (c4) edge [loop right] node {wait} (c4);
	\end{tikzpicture}
}%
\subfigure{
	\begin{tikzpicture}[->,>=stealth',shorten >=1pt,auto,node distance=2cm]
	  \tikzstyle{every state}=[circle,thick,draw=blue!75,minimum size=6mm]

	   \node [initial,initial text=,state] (c1) at (0,2.4) {$c_1$};
	   \node [state] (c2) at (0,1.2){$c_2$};
	   \node [state,label=below:$\BFM_1 \ldots \BFM_4$] (c3) at (0,0) {$c_3$};
	   
	   \path (c1) edge [loop right] node {$0.6$} (c1)
	                     edge node [] {$0.4$} (c2)
	             (c2) edge [loop right] node {$0.2$} (c2)
	                     edge node [] {$0.8$} (c3)
	             (c3) edge [loop right] node {$1$} (c3);
	\end{tikzpicture}
}%
\subfigure{
	\begin{tikzpicture}[->,>=stealth',shorten >=1pt,auto,node distance=2cm]
	  \tikzstyle{every state}=[circle,thick,draw=blue!75,minimum size=6mm]

	   \node [initial,initial text=,state] (c1) at (0,2.4) {$c_1$};
	   \node [state] (c2) at (0,1.2){$c_2$};
	   \node [state,label=below:$\BFM_5$] (c3) at (0,0) {$c_3$};
	   
	   \path (c1) edge [loop right] node {$0.6$} (c1)
	                     edge [bend left] node {$0.4$} (c2)
	             (c2) edge [loop right] node {$0.2$} (c2)
	                     edge [bend left]  node{$0.4$} (c3)
	                     edge [bend left] node {$0.4$} (c1)
	             (c3) edge [loop right] node {$0.6$} (c3)
	                     edge [bend left] node {$0.4$} (c2);
	\end{tikzpicture}
}
   \caption{TS $\BFT$ and MCs $\BFM_1\ldots\BFM_5$ that model the car and the pedestrians.}
   \label{fig:simple_models}
\end{figure}
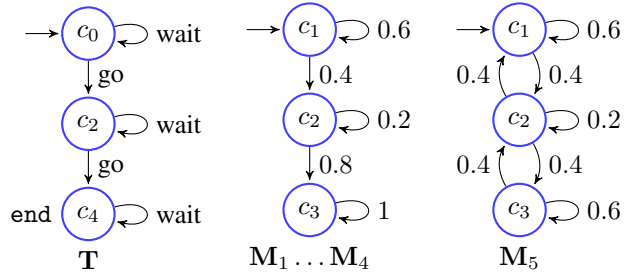

\begin{figure}
\centering 
\begin{tikzpicture}[->,>=stealth',shorten >=1pt,auto,node distance=2cm]
  \tikzstyle{every state}=[circle,thick,draw=blue!75,minimum size=6mm]

   \node [initial,initial text=,state,label=above:$\BFF$] (q0) at (0,0) {$q_0$};
   \node [state, double, ] (q1) at (2,0){$q_1$};
   \node [state] (q2) at (4,0) {$q_2$};
   
   \path (q0) edge [loop below] node {$\notltl \mathtt{col} \andltl \notltl \mathtt{end}$} (q0)
              edge node {$\mathtt{end}$} (q1) 
			  edge [bend left] node {$\mathtt{col} \andltl \notltl \mathtt{end}$} (q2)
             (q1) edge [loop below] node {$\True$} (q1)
             (q2) edge [loop below] node {$\True$} (q2);
\end{tikzpicture}
	\caption{Deterministic FSA $\BFF$ that corresponds to $\phi=\notltl \mathtt{col}\,\Until\,\mathtt{end}$ where $\mathtt{col}=\bigvee_{i=1\ldots5,j=0\ldots4} (c^\RMT_j\andltl c^i_j)$ and $\mathtt{end} = c^\RMT_4$. $q_0$ and $q_1$ are initial and final states, respectively.}
   \label{fig:simple_fsa}
\end{figure}
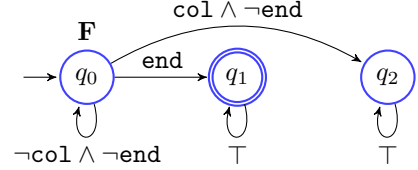

\subsection{Solution Outline\label{sec:sub:sol}}

One can directly solve Prob.~\ref{prb:problem} by reducing it to a Maximal Reachability Probability (MRP) problem on the MDP modeling the overall system~\cite{LDA:97}.
This approach, however, is very resource demanding as it scales exponentially with the number agents.
As a result, the environment size and the number of agents that can be handled in a reasonable time frame and with limited memory are small.
To address this issue, we propose a highly efficient incremental control synthesis method that exploits the independence between the system components and the fact that verification is less demanding than synthesis.
At each iteration $i$, our method will involve the following steps: synthesis of an optimal control policy considering only some of the agents (Sec.~\ref{sec:sub:syn}), verification of this control policy with respect to the complete system (Sec.~\ref{sec:sub:ver}) and minimization of the system model under the guidance of this policy (Sec.~\ref{sec:sub:min}).

\section{Problem Solution \label{sec:sol}}

Our solution to Prob.~\ref{prb:problem} is given in the form of Alg.~\ref{alg:ics}.
In the rest of this section, we explain each of its steps in detail.

\setlength{\algomargin}{1.5em}
\begin{algorithm}[h]
\DontPrintSemicolon
\SetInd{0.5em}{0.5em}
\KwIn{$\BFT,\BFM_1,\ldots,\BFM_n, \phi, (p_{thr}).$}
\KwOut{$\mu^\star$ s.t. $P^{\mu^\star}(\phi) \geq P^{\mu}(\phi)\,\forall \mu$ if $p_{thr}$ is not given, otherwise $P^{\mu^\star}(\phi) > p_{thr}$.}
$\CM \leftarrow \{\BFM_1,\ldots,\BFM_n\}$.\;
Construct FSA $\BFF$ corresponding to $\phi$.\;
$\mu^\star\leftarrow\emptyset,\,P_\CM^{\mu^\star}(\phi)\leftarrow 0,\,\CM_0 \leftarrow \emptyset,\,\BFA_0 \leftarrow \BFT,\,i\leftarrow 1$.\;
Process $\phi$ to form $\CM_i^{new}$.\;
\While{True} {
$\CM_i \leftarrow \CM_{i-1} \cup \CM_i^{new}$.\;
$\BFA_i \leftarrow \BFA_{i-1} \otimes \CM_i^{new}$.\;
$\BFP_i \leftarrow \BFA_i \otimes \BFF$.\;
Synthesize $\mu_i$ that maximizes $P^{\mu_i}_{\CM_i}(\phi)$ using $\BFP_i$.\;
\If{$p_{thr}$ given}{
\If{$P^{\mu_i}_{\CM_i}(\phi) < p_{thr}$}
{Fail: $\nexists \mu$ such that $P^{\mu}(\phi) \geq p_{thr}$.}
\ElseIf{$\CM_i = \CM$}
{Success: $\mu^\star \leftarrow \mu_i$, Return $\mu^\star$.}
\Else{Continue with verification on line 20.}
}
\ElseIf{$\CM_i = \CM$} {
{Success: $\mu^\star \leftarrow \mu_i$, Return $\mu^\star$.}
}
\Else{
Obtain the MC $\BFM^{\mu_i}_{\CM_i}$ induced on $\BFP_i$ by $\mu_i$.\;
$\overline{\CM_i} \leftarrow \CM \setminus \CM_i$.\;
$\BFM^{\mu_i}_{\CM} \leftarrow \BFM^{\mu_i}_{\CM_i}\otimes \overline{\CM}_i$.\;
$\BFV_i \leftarrow \BFM^{\mu_i}_{\CM}\otimes\BFF$.\;
Compute $P^{\mu_i}_\CM(\phi)$ using $\BFV_i$.\;
\If{$P^{\mu_i}_\CM(\phi) > P^{\mu^\star}_\CM(\phi)$}{
$\mu^\star \leftarrow \mu_i,\,P^{\mu^\star}_\CM(\phi)\leftarrow P^{\mu_i}_\CM(\phi)$.
}
\If{$p_{thr}$ given and $P_\CM^{\mu^\star}(\phi)>p_{thr}$} {
	Success: Return $\mu^\star$.
}
\Else{
	Set $\CM_{i+1}^{new}$ to some agent $\BFM_j \in \overline{\CM_i}$.\;
	Minimize $\BFA_i$.\;
	Increment $i$.\;
}
}
}
\caption{{\sc Incremental-Control-Synthesis}\label{alg:ics}}
\end{algorithm}

\subsection{Parallel Composition of System Components\label{sec:sub:comp}}

Given the set $\CM=\{\BFM_1,\ldots,\BFM_n\}$ of all agents, we use $\CM_i\subseteq\CM$ to denote its subset used at iteration $i$.
Then, we define the synchronous parallel composition $\BFT\otimes\CM_i$ of $\BFT$ and agents in $\CM_i=\{\BFM_{i1},\ldots,\BFM_{ij}\}$ for different types of $\BFT$ as follows.

If $\BFT$ is a TS, then we define $\BFT\otimes\CM_i$ as the MDP $\BFA = (\CQ_\RMA, q_\RMA^0, \CA_\RMA, \alpha_\RMA, \delta_\RMA, \Pi_\RMA, \CL_\RMA) = \BFT \otimes \CM_i$, such that
\begin{itemize}
\item $\CQ_\RMA \subseteq \CQ_\RMT \times \CQ_{i1} \times \ldots \times \CQ_{ij}$ such that a state $q=\langle q_\RMT,q_{i1},\ldots,q_{ij}\rangle$ exists iff it is reachable from the initial states;
\item $q_\RMA^0 = \langle q_\RMT^0, q_{i1}^0, \ldots,q_{ij}^0\rangle$;
\item $\CA_\RMA = \CA_\RMT$;
\item $\alpha_\RMA(q) = \alpha_\RMT(q_\RMT)$, where $q_\RMT$ is the element of $q$ that corresponds to the state of $\BFT$;
\item $\Pi_\RMA = \Pi_\RMT\cup\Pi_{i1}\cup\ldots\cup\Pi_{ij}$;
\item $\CL_\RMA(q) = \CL_\RMT(q_\RMT) \cup \CL_{i1}(q_{i1}) \cup \ldots\cup \CL_{ij}(q_{ij})$;
\item $\delta_\RMA(q=\langle q_\RMT,q_{i1},\ldots,q_{ij}\rangle,a,q'=\langle q_\RMT',q_{i1}',\ldots,q_{ij}'\rangle) = 1\{(q_\RMT,a,q_\RMT') \in \delta_T\}\times\delta(q_{i1},q_{i1}')\times\ldots\times\delta(q_{ij},q_{ij}')$,
\end{itemize}
where $1\{\cdot\}$ is the indicator function.

If $\BFT$ is an MDP, then we define $\BFT\otimes\CM_i$ as the MDP $\BFA = (\CQ_\RMA, q_\RMA^0, \CA_\RMA, \alpha_\RMA, \delta_\RMA, \Pi_\RMA, \CL_\RMA) = \BFT \otimes \CM_i$, such that $\CQ_\RMA$, $q_\RMA^0$, $\CA_\RMA$, $\alpha_\RMA$, $\Pi_\RMA$, and $\CL_\RMA$ are as given in the case where $\BFT$ is a TS and
\begin{itemize}
	\item $\delta_\RMA(q=\langle q_\RMT,q_{i1},\ldots,q_{ij}\rangle,a,q'=\langle q_\RMT',q_{i1}',\ldots,q_{ij}'\rangle) = \delta_\RMT(q_\RMT,a,q_\RMT')\times\delta_{i1}(q_{i1},q_{i1}')\times\ldots\times\delta_{ij}(q_{ij},q_{ij}')$.
\end{itemize}

Finally if $\BFT$ is an MC, then we define $\BFT\otimes\CM_i$ as the MC $\BFA = (\CQ_\RMA, q_\RMA^0, \delta_\RMA, \Pi_\RMA, \CL_\RMA) = \BFT \otimes \CM_i$ where $\CQ_\RMA$, $q_\RMA^0$, $\Pi_\RMA$, $\CL_\RMA$ are as given in the case where $\BFT$ is a TS and
\begin{itemize}
	\item $\delta_\RMA(q=\langle q_\RMT,q_{i1},\ldots,q_{ij}\rangle,q'=\langle q_\RMT',q_{i1}',\ldots,q_{ij}'\rangle) = \delta_\RMT(q_\RMT,q_\RMT')\times\delta_{i1}(q_{i1},q_{i1}')\times\ldots\times\delta_{ij}(q_{ij},q_{ij}')$.
\end{itemize}

\subsection{Product MDP and Product MC\label{sec:sub:prod}}

Given the deterministic FSA $\BFF$ that recognizes all and only the finite words that satisfy $\phi$, we define the product of $\BFM\otimes\BFF$ for different types of $\BFM$ as follows.

If $\BFM$ is an MDP, we define $\BFM\otimes\BFF$ as the product MDP $\BFP = (\CQ_\RMP, q_\RMP^0, \CA_\RMP, \alpha_\RMP, \delta_\RMP, \Pi_\RMP, \CL_\RMP) = \BFM \otimes \BFF$, where
\begin{itemize}
\item $\CQ_\RMP \subseteq \CQ_\RMM \times \CQ_\RMF$ such that a state $q$ exists iff it is reachable from the initial states;
\item $q_\RMP^0 = \langle q_\RMM^0,q_\RMF\rangle$ such that $(q_\RMF^0,\CL_\RMM(q_\RMM),q_\RMF)\in\delta_\RMF$;
\item $\CA_\RMP = \CA_\RMM$;
\item $\alpha_\RMP(\langle q_\RMM,q_\RMF\rangle) = \alpha_\RMM(q_\RMM)$;
\item $\Pi_\RMP = \Pi_\RMM$;
\item $\CL_\RMP(\langle q_\RMM,q_\RMF\rangle) = \CL_\RMM(q_\RMM)$;
\item $\delta_\RMP(\langle q_\RMM,q_\RMF\rangle,a,\langle q_\RMM',q'_\RMF\rangle) = 1\{(q_\RMF,\CL_\RMM(q_\RMM'),q_\RMF')\in\delta_\RMF\}$ $\times \delta_\RMM(q_\RMM,a,q_\RMM')$,
\end{itemize}
where $1\{\cdot\}$ is the indicator function.
In this product MDP, we also define the set $\CF_\RMP$ of final states such that a state $q=\langle q_\RMM,q_\RMF\rangle \in \CF_\RMP$ iff $q_\RMF \in \CF_\RMF$, where $\CF_\RMF$ is the set of final states of $\BFF$.

If $\BFM$ is an MC, we define $\BFM\otimes\BFF$ as the product MC $\BFP = (\CQ_\RMP, q_\RMP^0, \delta_\RMP, \Pi_\RMP, \CL_\RMP) = \BFM \otimes \BFF$ where $\CQ_\RMP$, $q_\RMP^0$, $\Pi_\RMP$, $\CL_\RMP$ are as given in the case where $\BFM$ is an MDP and 
\begin{itemize}
	\item $\delta_\RMP(\langle q_\RMM,q_\RMF\rangle,\langle q_\RMM',q'_\RMF\rangle) = 1\{(q_\RMF,\CL_\RMM(q_\RMM'),q_\RMF')\in\delta_\RMF\}\times\delta_\RMM(q_\RMM,q_\RMM')$.
\end{itemize}
In this product MC, we also define the set $\CF_\RMP$ of final states as given above.

\subsection{Initialization\label{sec:sub:init}}

Lines 1 to 4 of Alg.~\ref{alg:ics} correspond to the initialization procedure of our algorithm.
First, we form the set $\CM=\{\BFM_1,\ldots,\BFM_n\}$ of all agents and construct the FSA $\BFF$ that corresponds to $\phi$.
Such $\BFF$ can be automatically constructed using existing tools, \eg \cite{Latvala:EMC2003}.
Since we have not synthesized any control policies so far, we reset the variable $\mu^\star$ that holds the best policy at any given iteration and set the probability $P^{\mu^\star}_\CM(\phi)$ of satisfying $\phi$ under policy $\mu^\star$ in the presence of agents in $\CM$ to 0.
As we have not considered any agents so far, we set the subset $\CM_0$ to be an empty set.
We then set $\BFA_0$, which stands for the parallel composition of the plant $\BFT$ and the agents in $\CM_0$, to $\BFT$.
We also initialize the iteration counter $i$ to 1.

Line 4 of Alg.~\ref{alg:ics} initializes the set $\CM_1^{new}$ of agents that will be considered in the synthesis step of the first iteration of our algorithm.
In order to be able to guarantee completeness, we require this set to be the maximal set of agents that satisfy the mission, \ie the agent subset that can satisfy $\phi$ but not strictly needed to satisfy $\phi$.
To form $\CM_1^{new}$, we first rewrite $\phi$ in positive normal form to obtain $\phi_{pnf}$, where the negation operator $\notltl$ occurs only in front of atomic propositions.
Conversion of $\phi$ to $\phi_{pnf}$ can be performed automatically using De Morgan's laws and equivalences for temporal operators as given in \cite{CB-JPK:08}.
Then, using this fact, we include an agent $\BFM_i\in\CM$ in $\CM_1^{new}$ if any of its corresponding propositions of the form $(i,p),p \in \Pi_i$ appears non-negated in $\phi_{pnf}$.
For instance, given $\phi := \notltl ((3,\mathtt{p3})\andltl(\RMT,\mathtt{p3}))\,\Until\,((1,\mathtt{p1}) \orltl (2,\mathtt{p2}))$, either one of agents $\BFM_1$ and $\BFM_2$ can satisfy the formula, whereas agent $\BFM_3$ can only violate it.
Therefore, for this example we set $\CM_1^{new} = \{\BFM_1,\BFM_2\}$.
In case $\CM_1^{new}=\emptyset$ after this procedure, we form $\CM_1^{new}$ arbitrarily by including some agents from $\CM$ and proceed with the synthesis step of our approach.

\subsection{Synthesis\label{sec:sub:syn}}
Lines 6 to 19 of Alg.~\ref{alg:ics} correspond to the synthesis step of our algorithm.
At the $i^{th}$ iteration, the agent subset that we consider is given by $\CM_i = \CM_{i-1} \cup \CM_i^{new}$ where $\CM_i^{new}$ contains the agents that will be newly considered as provided by the previous iteration's verification stage or by the initialization procedure given in Sec.~\ref{sec:sub:init} if $i$ is 1.
First, we construct the parallel composition $\BFA_i = \BFA_{i-1} \otimes \CM_i^{new}$ of our plant and the agents in $\CM_i$ as described in Sec.~\ref{sec:sub:comp}.
Notice that, we use $\BFA_{i-1}$ to save from computation time and memory as $\BFA_{i-1} \otimes \CM_i^{new}$ is typically smaller than $\BFT \otimes \CM_i$ due to the minimization procedure explained in Sec.~\ref{sec:sub:min}.
Next, we construct the product MDP $\BFP_i = \BFA_i \otimes \BFF$ as explained in Sec.~\ref{sec:sub:prod}.
Then, our control synthesis problem can be solved by solving a maximal reachability probability (MRP) problem on $\BFP_i$ where one computes the maximum probability of reaching the set $\CF_\RMP$ from the initial state $q_\RMP^0$~\cite{LDA:97}, after which the corresponding optimal control policy $\mu_i$ can be recovered as given in \cite{CB-JPK:08,DiSmBeRu-CDC-2011}.
Consequently, at line 9 of Alg.~\ref{alg:ics} we solve the MRP problem on $\BFP_i$ using value iteration to obtain optimal policy $\mu_i$ that maximizes the probability of satisfaction of $\phi$ in the presence of the agents in $\CM_i$.
We denote this probability by $P_{\CM_i}^{\mu_i}(\phi)$, whereas $P^{\mu}(\phi)$ stands for the probability that the complete system satisfies $\phi$ under policy $\mu$.

The steps that we take at the end of the synthesis, \ie lines 10 to 19 of Alg.~\ref{alg:ics}, depends on whether $p_{thr}$ is given or not.
At any iteration $i$, if $p_{thr}$ is given and $P^{\mu_i}_{\CM_i}(\phi) < p_{thr}$, we terminate by reporting that there exists no control policy $\mu : P^\mu(\phi) \geq p_{thr}$ which is a direct consequence of Prop.~\ref{prp:syn}.
If $p_{thr}$ is given and $P^{\mu_i}_{\CM_i}(\phi)\geq p_{thr}$, we consider the following cases.
If $\CM_i=\CM$, we set $\mu^\star$ to $\mu_i$ and return $\mu^\star$ as it satisfies the probability threshold.
Otherwise, we proceed with the verification of $\mu_i$ as there are remaining agents that were not considered during synthesis and can potentially violate $\phi$.
For the case where $p_{thr}$ is not given we consider the current agent subset $\CM_i$.
If $\CM_i=\CM$ we terminate and return $\mu^\star$ as there are no agents left to consider.
Otherwise, we proceed with the verification stage.

\begin{proposition}
\label{prp:syn}
The sequence $\{P^{\mu_i}_{\CM_i}(\phi)\}$ is non-increasing.
\end{proposition}
\begin{proof}
As given in Sec.~\ref{sec:sub:init}, $\CM_1$ includes all those agents that can satisfy the propositions that lead to satisfaction of $\phi$.
Let $pref(\phi)$ be the set of finite words that satisfy $\phi$ and let MC $\BFM_j$ of agent $j$ be such that $\BFM_j \not \in \CM_1$.
Consider a finite satisfying word $\sigma$ such that $\sigma = \sigma^0 \sigma^1 \ldots \sigma^l \in pref(\phi)$.
Suppose there exists an index $k \in \{0,\ldots,l\}$ such that for some $q \in \CQ_j$ and $\CL_j(q) \in \sigma^k$.
Then, $\tilde{\sigma} = \sigma^0 \sigma^1 \ldots \sigma^{k-1} \tilde{\sigma}^k \sigma^{k+1} \ldots \sigma^l$ is also in $pref(\phi)$ where $\tilde{\sigma}^k = \sigma^k \setminus \CL_j(q)$.
Now, let $r = q^0 q^1 \ldots q^l$ be a path of the system after including $\BFM_j$.
Let $\omega = \CL(r) = \omega^0 \omega^1 \ldots \omega^l$ be the word generated by $r$.
If $\omega$ satisfies $\phi$, then $\tilde{\omega} = \tilde{\omega}^0 \tilde{\omega}^1 \ldots \tilde{\omega}^l$ also satisfies $\phi$ where $\tilde{\omega}^k = \omega^k \setminus \mathcal{L}_j(q_j^k)$ for each $k\in\{0,\ldots,l\}$ and $q_j^k$ is the state of $\BFM_j$ in $q^k$.
Thus, we conclude that the set of paths that satisfy $\phi$ cannot increase after we add agent $\BFM_j \in \CM\setminus\CM_1$, and the sequence $\{P^{\mu_i}_{\CM_i}(\phi)\}$ is non-increasing such that it attains its maximum value $P^{\mu_1}_{\CM_1}(\phi)$ at the first iteration and does not increase as more agents from $\CM\setminus\CM_1$ are considered in the following iterations.
\end{proof}

\begin{corollary}
	If at any iteration $P^{\mu_i}_{\CM_i}(\phi) < p_{thr}$, then there does not exist a policy $\mu : P^\mu(\phi) \geq p_{thr}$, where $\mu_i$ is an optimal control policy that we compute at the synthesis stage of the $i^{th}$ iteration considering only the agents in $\CM_i$.
\end{corollary}

\subsection{Verification and Selection of $\CM_{i+1}^{new}$\label{sec:sub:ver}}
Lines 20 to 30 of Alg.~\ref{alg:ics} correspond to the verification stage of our algorithm.
In the verification stage, we verify the policy $\mu_i$ that we have just synthesized considering the entire system and accordingly update the best policy so far, which we denote by $\mu^\star$.

Note that $\mu_i$ maximizes the probability of satisfying $\phi$ in the presence of agents in $\CM_i$ and induces an MC by resolving all non-deterministic choices in $\BFP_i$.
Thus, we first obtain the induced Markov Chain $\BFM^{\mu_i}_{\CM_i}$ that captures the joint behavior of the plant and the agents in $\CM_i$ under policy $\mu_i$.
Then, we proceed by considering the agents that were not considered during synthesis of $\mu_i$, \ie agents in $\overline{\CM_i} = \CM \setminus \CM_i$.
In order to account for the existence of the agents that we newly consider, we exploit the independence between the systems and construct the MC $\BFM^{\mu_i}_{\CM} = \BFM^{\mu_i}_{\CM_i}\otimes \overline{\CM}_i$ in line 22.
In lines 23 and 24 of Alg.~\ref{alg:ics}, we construct the product MC $\BFV_i=\BFM^{\mu_i}_{\CM}\otimes\BFF$ and compute the probability $P^{\mu_i}_\CM(\phi)$ of satisfying $\phi$ in the presence of all agents in $\CM$ by computing the probability of reaching $\BFV_i$'s final states from its initial state using value iteration.
Finally, in lines 25 and 26 we update $\mu^\star$  so that $\mu^\star = \mu_i$ if $P^{\mu_i}_\CM(\phi) > P^{\mu^\star}_\CM(\phi)$, \ie if we have a policy that is better than the best we have found so far.
Notice that, keeping track of the best policy $\mu^\star$ makes Alg.~\ref{alg:ics} an anytime algorithm, \ie the algorithm can be terminated as soon as some $\mu^\star$ is obtained.

At the end of the verification stage, if $p_{thr}$ is given and $P_\CM^{\mu^\star}(\phi) \geq p_{thr}$ we terminate and return $\mu^\star$, as it satisfies the given probability threshold.
Otherwise in line 30 of Alg.~\ref{alg:ics}, we pick an arbitrary $\BFM_j \in \overline{\CM}_i$ to be included in $\CM^{i+1}$, which we call the {\em random agent first (RAF)} rule.
Note that, one can also choose to pick the smallest $\BFM_j$ in terms of state and transition count to minimize the overall computation time, which we call the {\em smallest agent first (SAF)} rule.

\begin{proposition}
\label{prp:ver}
The sequence $\{P_\CM^{\mu^\star}(\phi)\}$ is a non-decreasing sequence.
\end{proposition}
\begin{proof}
The result directly follows from the fact that $\mu^\star$ is set to $\mu_i$ if and only if $P^{\mu_i}_\CM(\phi) > P^{\mu^\star}_\CM(\phi)$.
\end{proof}

\subsection{Minimization\label{sec:sub:min}}
The minimization stage of our approach (line 31 in Alg.~\ref{alg:ics}) aims to reduce the overall resource usage by removing those transitions and states of $\BFA_i$ that are not needed in the subsequent iterations.
We first set the minimization threshold $p_{min}$ to $p_{thr}$ if given, otherwise we set it to $P^{\mu^\star}_\CM(\phi)$.
Next, we iterate over the states of $\BFP_i$ and check the maximum probability of satisfying the mission under each available action.
Note that, the value iteration that we perform in the synthesis step already provides us with the maximum probability of satisfying $\phi$ from any state in $\BFP_i$.
Then, we remove an action $a$ from state $q_\RMA$ in $\BFA_i$ if for all $q_\RMF \in Q_\RMF$, the maximum probability of satisfying the mission by taking action $a$ at $\langle q_\RMA, q_\RMF\rangle$ in $\BFP_i$ is below $p_{min}$.
After removing the transitions corresponding to all such actions, we also prune any orphan states in $\BFA_i$, \ie states that are not reachable from the initial state.
Then, we proceed with the synthesis stage of the next iteration.

\begin{proposition}
\label{prp:min}
Minimization phase does not affect the correctness and the completeness of our approach.
\end{proposition}
\begin{proof}
To prove the correctness, we need to show that for an arbitrary policy $\mu$ on the minimized MDP $\BFA_{min}$, the probability that $\BFA_{min}$ satisfies $\phi$ under $\mu$ is equal to the probability that $\BFA$ satisfies $\phi$ under $\mu$ where $\BFA$ is the original MDP before minimization.
Correctness, in this case, follows directly from the fact that, in each state $q$, we do not modify the transition probabilities associated with an action that is enabled in $q$ after minimization.
Thus, it remains to show that minimization does not affect the completeness of the approach.
We first consider the removal of orphaned states.
Since these states cannot be reached from the initial state, they also will not be a part of any feasible control policy, and their removal does not affect the completeness of the approach.
Finally, we consider the removal of those actions that drive the system to the set of target states with probability smaller than the minimization threshold.
For the case where we use $p_{thr}$, completeness is not affected as we remove only those transitions that we would not take as we are looking for control policies with $P^\mu_\CM(\phi) \geq p_{thr}$ and $P^{\mu_i}_{\CM_i}$ is a non-increasing sequence (Prop.~\ref{prp:syn}).
For the case where we use $P^{\mu^\star}_\CM$, we also remove those transitions that we would not take as $P^{\mu^\star}_\CM$ is a non-decreasing sequence (Prop.~\ref{prp:ver}).
Hence, the minimization procedure does not affect the completeness of the overall approach as well.
\end{proof}

We finally show that Alg.~\ref{alg:ics} correctly solves Prob.~\ref{prb:problem}.

\begin{proposition}
Alg.~\ref{alg:ics} solves Prob.~\ref{prb:problem}.
\end{proposition}
\begin{proof}
Alg.~\ref{alg:ics} combines all the steps given in this section and synthesizes a control policy $\mu^\star$ that either ensures $P^{\mu^\star}(\phi) \geq p_{thr}$ if $p_{thr}$ is given, or maximizes $P^{\mu^\star}(\phi)$.
If Alg.~\ref{alg:ics} terminates in line 12, completeness is guaranteed by the fact that $P^{\mu_i}_{\CM_i}$ is a non-increasing sequence as given in Prop.~\ref{prp:syn}.
Also, as given in Prop.~\ref{prp:min}, minimization stage does not affect the correctness and completeness of the approach.
Thus, Alg.~\ref{alg:ics} solves Prob.~\ref{prb:problem}.
\end{proof}

\section{Experimental Results\label{sec:impl}}

\begin{figure}
\subfigure{
	\centering
	\includegraphics[width=0.9\linewidth]{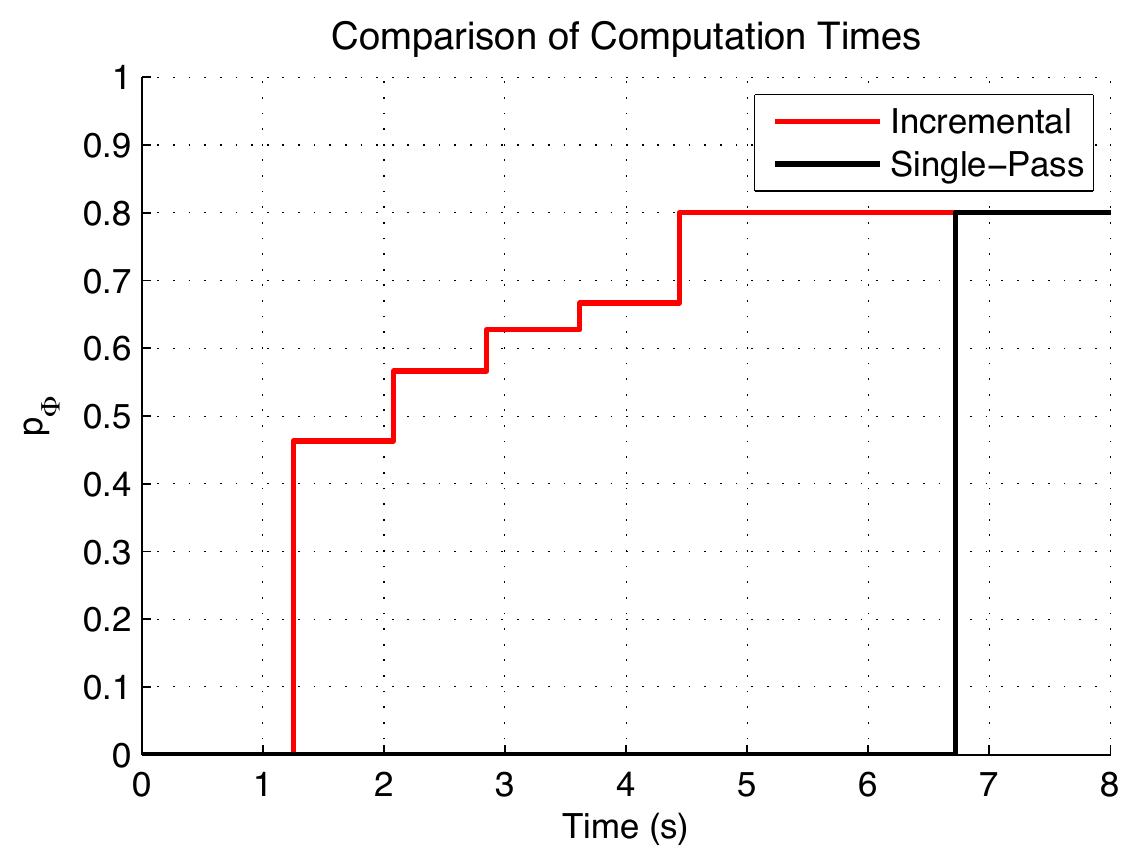}
}
\subfigure{
	\centering
	\includegraphics[width=0.9\linewidth]{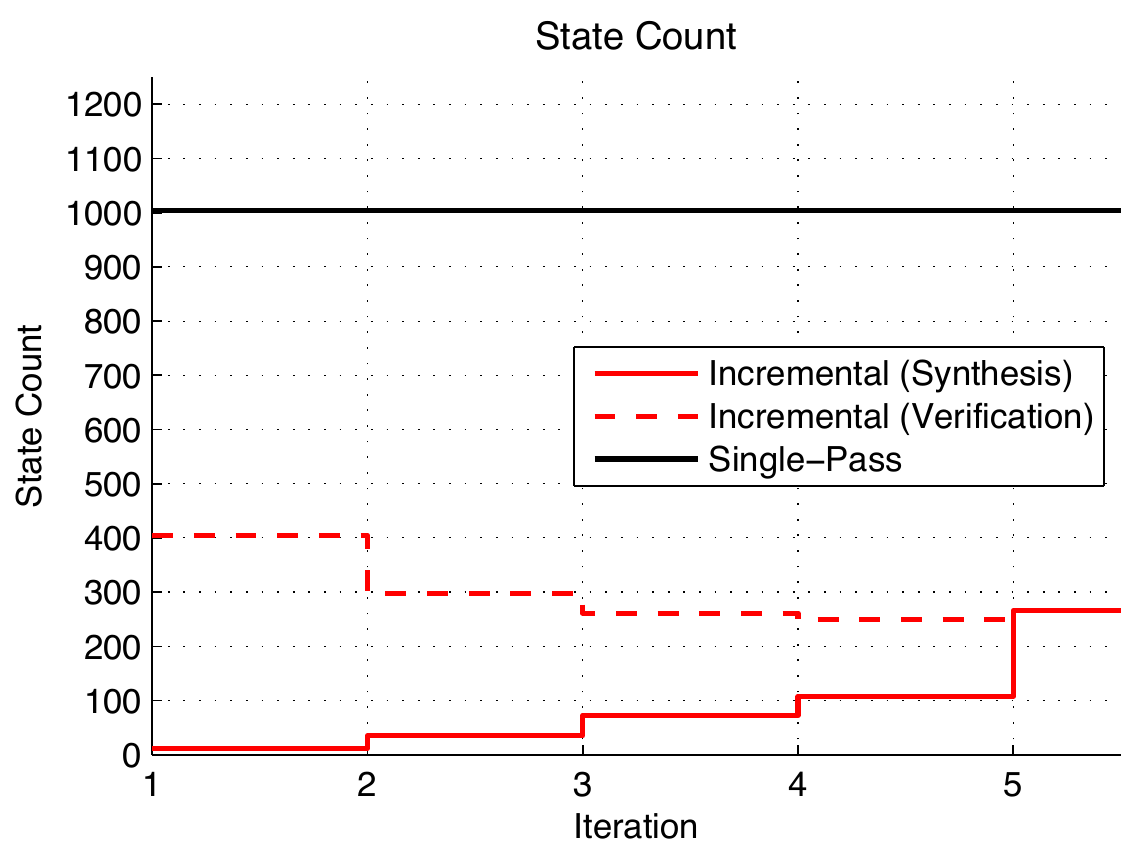}
}
\caption{Comparison of the classical single-pass and proposed incremental algorithms. The top plot shows the running times of the algorithms and the probabilities of satisfying $\phi$ under synthesized policies. The bottom plot compares the state counts of the product MDPs on which the MRP problem was solved in both approaches (black and red lines) and shows the state count of the product MC considered in the verification stage of our incremental algorithm (red dashed line).\label{fig:results}}
\end{figure}

In this section we return to the pedestrian crossing problem given in Example~\ref{exp:simple} and illustrated in Figs.~\ref{fig:drawing}, \ref{fig:simple_models}.
The mission specification $\phi$ for this example is given in Eq.~\eqref{eqn:phi}.
In the following, we compare the performance of our incremental algorithm with the performance of the classical method that attempts to solve this problem in a single pass using value iteration as in \cite{LDA:97}.

In our experiments we used an iMac i5 quad-core desktop computer and considered C++ implementations of both approaches.
During the experiments, our algorithm picked the new agent $\CM_i^{new}$ to be considered at the next iteration in the following order: $\BFM_1,\BFM_2,\BFM_3,\BFM_4,\BFM_5$, \ie according to the smallest agent first rule given in Sec.~\ref{sec:sub:ver}.

When no $p_{thr}$ was given, optimal control policies synthesized by both of the algorithms satisfied $\phi$ with a probability of 0.8.
The classical approach solved the control synthesis problem in 6.75 seconds, and the product MDP on which the MRP problem was solved had 1004 states and 26898 transitions.
In comparison, our incremental approach solved the same problem in 4.44 seconds, thanks to the minimization stage of our approach, which reduced the size of the problem at every iteration by pruning unneeded actions and states.
The largest product MDP on which the MRP problem was solved in the synthesis stage of our approach had 266 states and 4474 transitions.
The largest product MC that was considered in the verification stage of our approach had 405 states and 6125 transitions.
The probabilities of satisfying $\phi$ under policy $\mu_i$ obtained at each iteration of our algorithm were $P^{\mu_1}_\CM(\phi) = 0.463$, $P^{\mu_2}_\CM(\phi) = 0.566$, $P^{\mu_3}_\CM(\phi) = 0.627$, $P^{\mu_4}_\CM(\phi) = 0.667$, and $P^{\mu_5}_\CM(\phi) = 0.8$.
When $p_{thr}$ was given as 0.65, our approach finished in 3.63 seconds and terminated after the fourth iteration returning a sub-optimal control policy with a 0.667 probability of satisfying $\phi$.
In this case, the largest product MDP on which the MRP problem was solved had only 99 states and 680 transitions.
Furthermore, since our algorithm runs in an anytime manner, it could be terminated as soon as a control policy was available, \ie at the end of the first iteration (1.25 seconds).
Fig.~\ref{fig:results} compares the classical single-pass approach with our incremental algorithm in terms of running time and state counts of the product MDPs and MCs.

It is interesting to note that state count of the product MDP considered in the synthesis stage of our algorithm increases as more agents are considered, whereas state count of the product MC considered in the verification stage of our algorithm decreases as the minimization stage removes unneeded states and transitions after each iteration.
It must also be noted that, $|\CM_1|$, \ie cardinality of the initial agent subset, is an important factor for the performance of our algorithm.
As discussed in this section, for $|\CM_1| << |\CM|$ our algorithm outperforms the classical method both in terms of running time and memory usage.
However, for $|\CM_1|\sim|\CM|$ we expect the resource usage of our algorithm to be close to that of the classical approach, as in this case almost all of the agents will be considered in the synthesis stage of the first iteration.
We plan to address this issue in future work. %
Nevertheless, most typical finite horizon \emph{safety} missions, where the plant is expected to reach a goal while avoiding a majority or all of the agents, already satisfy the condition $|\CM_1| << |\CM|$.

\section{Conclusions \label{sec:conclusions}}
In this paper we presented a highly efficient incremental method for automatically synthesizing optimal control policies for a system comprising a plant and multiple independent agents, where the plant is expected to satisfy a high level mission specification in the presence of the agents.
We considered independent agents modeled as Markov chains and assumed that the plant was modeled as a deterministic transition system.
However, our approach is general enough to accommodate plants modeled as Markov Decision Processes.
For mission specifications, we considered syntactically co-safe Linear Temporal Logic formulas over a set of propositions that are satisfied by the components of the system.
If a probability threshold is given, our method exploits this knowledge to terminate earlier and returns a sub-optimal control policy.
Otherwise, our method synthesizes an optimal control policy that maximizes the probability of satisfying the mission.
Since our method does not need to run to completion, it has practical value in applications where a safe control policy must be synthesized under resource constraints.
For future work, we plan to extend our approach to mission specifications expressed in full LTL as opposed to a subset of it.

\bibliographystyle{IEEEtran}
\bibliography{references}

\end{document}